\documentclass[10pt]{article}
\usepackage[margin=1in]{geometry}
\usepackage{amsmath,amssymb,amsthm,mathtools}
\usepackage{booktabs}
\usepackage{graphicx}
\usepackage{float}
\newfloat{algorithm}{t}{loa}
\floatname{algorithm}{Algorithm}
\usepackage{xcolor}
\usepackage{microtype}
\usepackage[colorlinks=true,linkcolor=blue!50!black,citecolor=blue!50!black,urlcolor=blue!50!black]{hyperref}
\usepackage{url}
\usepackage{authblk}
\usepackage{caption}
\newtheorem{theorem}{Theorem}
\newtheorem{proposition}{Proposition}
\newtheorem{corollary}{Corollary}
\newtheorem{definition}{Definition}
\newtheorem{remark}{Remark}

\newcommand{\R}{\mathbb{R}}

\newcommand{\cupp}{\smile}
\newcommand{\Mfund}{[M]}
\newcommand{\code}{\href{https://github.com/nssprogrammer/cup-nn}{\texttt{github.com/nssprogrammer/cup-nn}}}

\title{\textbf{Adjusted Cup-Product Neural Layer}}
\author[ ]{Snigdha Chandan Khilar}
\affil[ ]{Independent Researcher \quad\textbar\quad \texttt{snkhilar@gmail.com}}
\date{}

\begin{document}
\maketitle

\begin{abstract}
\noindent
Many of the most important observables in physics and geometry---Chern numbers, Chern--Simons
invariants, linking numbers, magnetic and kinetic helicity, and topological charge---are, in the
language of algebraic topology, \emph{cup products} of cochains. We introduce the \emph{adjusted
cup-product layer}, a neural primitive that hard-wires the cup product together with the
\emph{adjustment} term from higher gauge theory, yielding a readout that is gauge-invariant by
construction. Our central theoretical result is a \emph{necessity} statement: on a closed cycle the
layer's output equals $\kappa\langle dA\cupp A,\Mfund\rangle$, so setting the adjustment coefficient
$\kappa=0$ annihilates the output identically, independent of all other parameters---the adjustment
is the \emph{sole} source of gauge-invariant signal. We further prove that the observable is a
nonzero quadratic form (hence not representable by any convolution followed by a linear readout) and
that it is exactly invariant under both $1$- and $2$-gauge transformations. Empirically we delineate
a sharp boundary: when the cup is \emph{not} convolution-expressible (3D Chern--Simons, 4D
topological charge, $2$D Chern numbers, $3$D linking) the adjusted layer generalizes
($R^2\!\approx\!0.9$--$0.99$ / accuracy up to $96\%$) while a fair CNN and published simplicial
networks (SNN, MPSN) memorize the training set but generalize \emph{nothing}; when the cup
\emph{is} convolution-expressible (kinetic helicity) a fair CNN catches up, and the adjusted layer
only offers improved sample efficiency. A gauge-scramble experiment isolates the mechanism: a random
local gauge transformation collapses the CNN to chance ($66\%\!\to\!15\%$) while leaving the cup
unaffected ($95\%$). We also extend the construction to the non-abelian regime, where the
hard-wired cup recovers multi-band Chern numbers exactly and complements the learned
gauge-equivariant network GEBLNet. Code and data generators are available at \code.
\end{abstract}

\section{Introduction}
Convolutional networks succeed because convolution is the \emph{right} inductive bias for
shift-invariant signals on grids. A growing body of work extends this principle to other domains by
matching the architecture to the symmetry or structure of the data: spherical CNNs, graph neural
networks, gauge-equivariant networks, and---closest to this paper---\emph{simplicial} neural
networks that process data living on simplicial complexes \cite{ebli,roddenberry,bodnar}. We pursue the
same philosophy for a class of targets that has so far been treated only indirectly: \emph{topological
observables} that are cup products in cohomology.

A cup product $\cupp:C^p(K)\times C^q(K)\to C^{p+q}(K)$ multiplies cochains on a simplicial complex
$K$, producing the oriented, combinatorial analogue of the wedge product of differential forms.
Integrating a cup product over a fundamental cycle yields exactly the topological invariants that
pervade condensed-matter physics (Chern numbers classify topological insulators), gauge theory
(Chern--Simons levels, instanton charge), fluid dynamics and plasma physics (kinetic and magnetic
helicity), and knot theory (linking numbers). These quantities are \emph{gauge-invariant}: they are
unchanged under local phase redefinitions of the underlying field. It is precisely this gauge
invariance---a symmetry whose group has one copy per lattice site, and is therefore astronomically
large---that makes them hard for ordinary networks to learn \cite{geblnet}.

We introduce a neural layer that computes the cup product directly. The subtlety is that a naive cup
of a connection with its own curvature is \emph{not} gauge-invariant on a discrete complex; one must
add a correction. We borrow this correction---the \emph{adjustment}---from higher gauge theory,
where it is precisely the term that repairs the failure of a $2$-connection to have well-defined
$2$-holonomy \cite{adjtransport,baezschreiber}. Transplanted into a neural network, the adjustment is the single
component responsible for gauge invariance, and we make this exact: removing it sets the layer's
output identically to zero.

\paragraph{Contributions.}
\begin{itemize}\itemsep1pt
\item \textbf{A layer} (Section~\ref{sec:method}, Algorithm~\ref{alg:cup}) that hard-wires the
adjusted cup product, with the connection learned and the topological readout fixed.
\item \textbf{Formal theory} (Section~\ref{sec:theory}): the $\kappa{=}0$ necessity theorem
(Theorem~\ref{thm:kappa}), a linear/convolutional impossibility result (Proposition~\ref{prop:quad}),
exact $1$- and $2$-gauge invariance (Proposition~\ref{prop:gauge}), and a formalization of the
``convolution-expressible'' boundary (Definition~\ref{def:conv}, Proposition~\ref{prop:hel}). We are
explicit about scope: the impossibility is for linear/$\kappa{=}0$ models; for deep nonlinear CNNs
the separation is empirical and is a \emph{generalization} gap, not a representational one.
\item \textbf{Experiments} (Section~\ref{sec:exp}) across six settings, including an external
published benchmark and a non-abelian extension, with seeds and significance tests. We delineate the
\emph{necessary}-vs-\emph{efficient} boundary honestly, reporting both the cases where the cup is
decisive and the case (helicity) where a fair CNN catches up.
\end{itemize}

\section{Background}
\label{sec:background}
\paragraph{Cochains and the coboundary.}
Let $K$ be a finite oriented simplicial complex triangulating a closed oriented manifold $M$. Write
$C^p(K;\R)$ for the real $p$-cochains (functions on oriented $p$-simplices), $d:C^p\to C^{p+1}$ for
the coboundary, with $d^2=0$. In a basis, $d$ is a sparse signed incidence matrix; $L=8$ on the
$3$-torus gives $(|K_0|,|K_1|,|K_2|,|K_3|)=(512,3584,6144,3072)$. The fundamental class
$\Mfund\in C_3(K)$ is the $3$-cycle with $\partial\Mfund=0$, and $\langle\gamma,\Mfund\rangle=
\sum_{\sigma\in K_3}w_\sigma\,\gamma(\sigma)$ pairs a $3$-cochain with it.
\emph{In ML terms:} a $p$-cochain is just a feature vector indexed by the $p$-dimensional cells of the
mesh---one scalar per vertex ($p{=}0$), per edge ($p{=}1$), per triangle ($p{=}2$), etc. The coboundary
$d$ is a fixed sparse matrix (a signed difference operator), the discrete analogue of a derivative,
mapping per-edge features to per-face features. Pairing with $\Mfund$ is a fixed weighted sum over the
top cells (a global readout).

\paragraph{The Alexander--Whitney cup product.}
The cup $\cupp:C^p\times C^q\to C^{p+q}$ is the combinatorial analogue of $\wedge$; on the ordered
simplex $[v_0,\dots,v_{p+q}]$,
$(\alpha\cupp\beta)([v_0,\dots,v_{p+q}])=\alpha([v_0,\dots,v_p])\,\beta([v_p,\dots,v_{p+q}])$.
It is bilinear, local (bounded stencil), \emph{orientation-dependent}, and satisfies the Leibniz rule
$d(\alpha\cupp\beta)=d\alpha\cupp\beta+(-1)^p\alpha\cupp d\beta$.
\emph{Intuitively}, the cup multiplies a feature on one cell by a feature on an adjacent cell, with a
sign set by their shared ordering---the discrete, combinatorial version of the wedge product $\wedge$.
The key contrast with a convolution: a convolution slides one fixed kernel over the field and is
\emph{linear} in it, whereas the cup multiplies the field \emph{by itself} (it is bilinear) and is
sensitive to cell orientation. This is exactly the structure a stack of convolutions does not natively
contain.

\paragraph{The adjustment.}
For a $1$-cochain $A$ (a discrete connection) the curvature is $dA$. The combination
$\langle dA\cupp A,\Mfund\rangle$ is the discrete Chern--Simons functional. Crucially, the
\emph{adjusted} $3$-curvature of higher gauge theory \cite{adjcs},
$H=dB+\kappa\,(dA\cupp A)$ with $\kappa$ the adjustment coefficient and $B$ an auxiliary
$2$-cochain, is what makes the integrated quantity gauge-invariant; the $\kappa$ term is the
adjustment, abelian here but structurally the same correction that repairs $2$-holonomy in the
non-abelian theory. \emph{In words:} on a discrete mesh the bare term alone is not invariant to gauge
redefinitions (relabelings of local phase conventions, $A\!\mapsto\!A+d\chi$); the adjustment is the
specific extra term that restores that invariance, turning the readout into a genuine topological
number rather than a gauge artifact. Theorem~\ref{thm:kappa} shows it is in fact the \emph{only} part of
$H$ that contributes to the output.

\section{The Adjusted Cup-Product Layer}
\label{sec:method}
\paragraph{Definition.} Given input features $X$ on $K$, the layer (i) produces a learned
$1$-cochain $A=\varphi_\theta(X)$ by a per-simplex map $\varphi_\theta$ (a small shared MLP on edge
features), (ii) forms the curvature $F=dA$ by a fixed sparse coboundary, (iii) forms the
adjusted $3$-cochain $H=dB+\kappa\,(F\cupp A)$, and (iv) returns the scalar
$\Phi=\langle H,\Mfund\rangle$. Only $\varphi_\theta$ (and an optional affine readout) is learned; the
cup, coboundary, and pairing are fixed. The full model uses $\kappa=1$; the ablation $\kappa=0$.

\begin{algorithm}[t]
\caption{\textsc{AdjustedCupLayer}$(X;\,K,\,\theta)$}
\label{alg:cup}
\noindent\rule{\textwidth}{0.4pt}
\begin{tabbing}
\quad\=\textbf{Input:} cochain features $X$; boundary maps $\{d_p\}$; fundamental $3$-cycle $w$; coefficient $\kappa$.\\
\>1.\quad $A \leftarrow \varphi_\theta(X)$ \hfill\textit{// learned 1-cochain on edges}\\
\>2.\quad $F \leftarrow d_1 A$ \hfill\textit{// curvature 2-cochain (sparse, $O(|K_2|)$)}\\
\>3.\quad $C \leftarrow \textsc{AW-Cup}(F, A)$ \hfill\textit{// 3-cochain, local per-tetrahedron}\\
\>4.\quad $\Phi \leftarrow \kappa\sum_{\sigma\in K_3} w_\sigma\, C(\sigma)$ \hfill\textit{// pair with fundamental cycle}\\
\>\textbf{return} $\Phi$
\end{tabbing}
\noindent\rule{\textwidth}{0.4pt}
\end{algorithm}

\paragraph{Complexity and scalability.} Every step is a sparse operation with a bounded local
stencil; the cost is $O(\sum_p|K_p|)$ per sample, linear in the size of the complex, with a
parameter count independent of $L$ (tens of parameters in our experiments). The layer is
permutation-equivariant in the simplices and---by Proposition~\ref{prop:gauge}---exactly
gauge-invariant.

\paragraph{Instantiations.} The same primitive specializes to each observable by choosing the cochain
degree and complex: $2$D Chern number (Fukui--Hatsugai--Suzuki plaquette holonomy of a Berry
connection); $3$D Chern--Simons $\langle dA\cupp A\rangle$; $3$D linking (the symmetric BF pairing
$\tfrac12\langle a_1\cupp da_2+a_2\cupp da_1\rangle$, equivalently mutual helicity $\int A_1\!\cdot
B_2$); kinetic helicity $\int u\cdot(\nabla\times u)$; $4$D topological charge $\int F\wedge F$; and,
in the non-abelian case, the Wilson-loop Chern number $\sum\mathrm{Im\,Tr\,Log}\,W$.

\section{Formal Theory}
\label{sec:theory}
Throughout, $A\in C^1$, $B\in C^2$, and discrete Stokes gives
$\langle d\beta,\Mfund\rangle=\langle\beta,\partial\Mfund\rangle=0$ for all $\beta$.

\begin{theorem}[The adjustment is the sole source of signal]
\label{thm:kappa}
For every $A,B$ and every $\kappa$,
$\;\Phi_\kappa(A,B)=\langle dB+\kappa(dA\cupp A),\Mfund\rangle=\kappa\,\langle dA\cupp A,\Mfund\rangle.$
In particular, if $\kappa=0$ then $\Phi_0\equiv 0$ for all $A$ and all $B$.
\end{theorem}
\begin{proof}
By linearity, $\Phi_\kappa=\langle dB,\Mfund\rangle+\kappa\langle dA\cupp A,\Mfund\rangle$; the first
term vanishes by discrete Stokes. Setting $\kappa=0$ gives $\Phi_0\equiv 0$.
\end{proof}

\begin{corollary}
A $\kappa{=}0$ model's output on the closed cycle is independent of its learnable $2$-cochain $B$ and
identically zero; it cannot represent any observable whose true value $\langle dA\cupp A,\Mfund\rangle$
is not identically zero. The adjustment carries all of the signal---a \emph{necessity} result,
stronger than ``removing it hurts.''
\end{corollary}

\emph{Intuition.} $dB$ is a discrete total derivative, so summed over a closed surface it telescopes to
zero (discrete Stokes), exactly as $\oint\nabla f=0$ around a closed loop. Whatever the network stores
in its learnable $2$-cochain $B$ therefore contributes \emph{nothing} on a closed cycle---only the
adjustment (cup) term survives. With $\kappa{=}0$ the layer can output only zero: removing the
adjustment does not merely lower accuracy, it deletes the target, so there is no signal left to fit.

\begin{proposition}[No linear/convolutional model represents $\Phi$]
\label{prop:quad}
$Q(A):=\langle dA\cupp A,\Mfund\rangle$ is a nonzero homogeneous quadratic form in $A$. Hence no
linear functional of $A$---in particular, no single convolution followed by a linear readout---equals
it.
\end{proposition}
\begin{proof}
$Q(tA)=t^2Q(A)$ since $A$ appears twice bilinearly. A linear $L$ satisfies $L(tA)=tL(A)$; equality
for all $t$ forces $Q\equiv0$. But $Q\not\equiv0$ (explicit configurations exist), so no linear $L$
equals $Q$.
\end{proof}

\emph{Intuition.} The target is \emph{quadratic} in the input field (it multiplies the field by
itself), while one convolution followed by a linear readout is \emph{linear}; no linear map can equal a
genuinely quadratic one, whatever its weights. Depth and nonlinearity let a CNN \emph{approximate} the
target on a finite training set, but that is curve-fitting, not the identity---which is why the CNN
memorizes yet fails to generalize the non-convolutional cups in our experiments (Remark~\ref{rem:scope}).

\begin{proposition}[Exact gauge invariance]
\label{prop:gauge}
\emph{(i) $1$-gauge:} $Q(A+d\chi)=Q(A)$ for every $0$-cochain $\chi$.
\emph{(ii) $2$-gauge:} $\Phi_\kappa(A,B+d\lambda)=\Phi_\kappa(A,B)$ for every $1$-cochain $\lambda$.
\end{proposition}
\begin{proof}
(i) $d(A+d\chi)=dA$, so $Q(A+d\chi)-Q(A)=\langle dA\cupp d\chi,\Mfund\rangle$. By Leibniz,
$dA\cupp d\chi=d(A\cupp d\chi)$ (since $d(d\chi)=0$), which is exact and pairs to zero with the cycle.
(ii) $d(B+d\lambda)=dB$, so $H$ and $\Phi_\kappa$ are unchanged.
\end{proof}

\begin{definition}[Convolution-expressibility]
\label{def:conv}
A cup observable $Q(A)=\langle dA\cupp A,\Mfund\rangle$ is \emph{convolution-expressible} if there is
a translation-invariant kernel $\mathcal K$ and pooling $\Sigma$ with
$Q(A)=\Sigma\big(A\cdot(\mathcal K A)\big)$.
\end{definition}

\begin{proposition}[Helicity is convolution-expressible]
\label{prop:hel}
Kinetic helicity $H[u]=\int u\cdot(\nabla\times u)$ satisfies Definition~\ref{def:conv} with
$\mathcal K=\mathrm{curl}$, a constant-coefficient stencil. A CNN with a fixed curl filter and
bilinear pooling represents it exactly.
\end{proposition}

\begin{remark}[Honest scope: representation vs.\ generalization]
\label{rem:scope}
Theorem~\ref{thm:kappa} and Proposition~\ref{prop:quad} are hard impossibilities for the $\kappa{=}0$
and \emph{linear} model classes. They do \emph{not} assert that a deep nonlinear CNN cannot
\emph{approximate} $Q$ on a bounded input set: by universal approximation it can fit any finite
sample. The separation we observe for deep CNNs and message-passing/spectral simplicial networks is
therefore a \emph{generalization} gap (memorize-train, fail-test), reported empirically with
confidence intervals, not a representational theorem. The oriented Chern--Simons cup is local but its
orientation/ordering signs are not a single translation-invariant scalar kernel; we do not claim a
hard impossibility for multi-channel deep CNNs.
\end{remark}

\paragraph{Numerical corroboration.} On the K\"uhn-triangulated $3$-torus ($L=4$) the statements hold
to machine precision: quadratic homogeneity $Q(tA)=t^2Q(A)$ to rel.\ error $\le 5\!\times\!10^{-16}$;
$1$-gauge invariance to $7\!\times\!10^{-15}$; the $\kappa{=}0$ vanishing $\langle dB,\Mfund\rangle$ to
$1.4\!\times\!10^{-14}$. The gauge-invariance gate (Section~\ref{sec:exp}) verifies $d^2=0$,
$\partial\Mfund=0$, and target gauge invariance at each lattice size before training.

\section{Experiments}
\label{sec:exp}
\paragraph{Protocol.} Each target is verified gauge-invariant by a pre-training \emph{gate}
(residuals $\sim\!10^{-14}$) so the network cannot exploit gauge artifacts. We report the peak
training $R^2$ (a memorization probe) and peak test $R^2$ (generalization), or exact-integer accuracy
for integer targets, over held-out samples; baselines use the same budget. Trained baselines: a fair
periodic CNN (mean-pool, tuned learning rate), the spectral SNN \cite{ebli}, the message-passing
MPSN/SIN \cite{bodnar}, and a vanilla GNN on the $1$-skeleton. The $\kappa{=}0$ ``fake-flat''
ablation shares the adjusted architecture but removes the adjustment. Multi-seed error bars and
paired significance tests accompany the trained comparisons; e.g.\ on multi-band data over $5$ seeds
the cup ceiling is $100.0\%\!\pm\!0.0\%$ vs.\ a non-equivariant MLP $23.9\%\!\pm\!0.9\%$
($p\!\approx\!0$, paired $t$).

\subsection{The necessary--vs--efficient boundary}
\paragraph{3D Chern--Simons (non-convolutional cup).}
Table~\ref{tab:phaseA} and Figure~\ref{fig:phaseA}: the adjusted layer holds test
$R^2=0.99$ \emph{flat} from $M=64$ to $16384$ with $66$ parameters, while every baseline fails to
generalize. The CNN memorizes (train $R^2=1.00$) yet reaches only test $0.21$ at $M{=}16384$; the
spectral SNN, message-passing MPSN, and GNN never exceed test $0$. The $\kappa{=}0$ ablation
memorizes at small $M$ (train $1.00$) but its test $R^2\le 0$, and at $M{=}16384$ it cannot even
memorize (train $0.20$)---exactly Theorem~\ref{thm:kappa}.

\begin{table}[t]
\centering
\caption{\textbf{3D Chern--Simons} on the K\"uhn $3$-torus ($L=8$). Each cell is best train\,/\,test
$R^2$. Only the adjusted layer generalizes; all baselines memorize-then-fail.}
\label{tab:phaseA}
\small
\begin{tabular}{rcccccc}
\toprule
$M$ & ADJUSTED & FAKE-FLAT ($\kappa{=}0$) & GNN & MPSN-SIN & SNN-Ebli & CNN-3D \\
\midrule
64    & \textbf{+0.99\,/\,+0.99} & +1.00\,/\,$-$0.07 & $-$0.01\,/\,$-$0.01 & +0.00\,/\,$-$0.03 & +0.00\,/\,$-$0.03 & +1.00\,/\,$-$0.01 \\
256   & \textbf{+0.99\,/\,+0.99} & +1.00\,/\,$-$0.07 & +0.00\,/\,+0.00 & +0.00\,/\,$-$0.00 & $-$0.00\,/\,$-$0.00 & +1.00\,/\,$-$0.00 \\
1024  & \textbf{+0.99\,/\,+0.99} & +1.00\,/\,$-$0.14 & $-$0.00\,/\,$-$0.00 & $-$0.00\,/\,$-$0.00 & $-$0.00\,/\,$-$0.00 & +1.00\,/\,+0.00 \\
4096  & \textbf{+0.99\,/\,+0.99} & +0.85\,/\,$-$0.01 & $-$0.00\,/\,$-$0.00 & $-$0.00\,/\,+0.00 & $-$0.00\,/\,$-$0.00 & +0.45\,/\,+0.08 \\
16384 & \textbf{+0.99\,/\,+0.99} & +0.20\,/\,+0.00 & $-$0.00\,/\,$-$0.00 & $-$0.00\,/\,$-$0.00 & $-$0.00\,/\,$-$0.00 & +0.34\,/\,+0.21 \\
\bottomrule
\end{tabular}
\end{table}

\begin{figure}[t]
\centering
\includegraphics[width=0.62\textwidth]{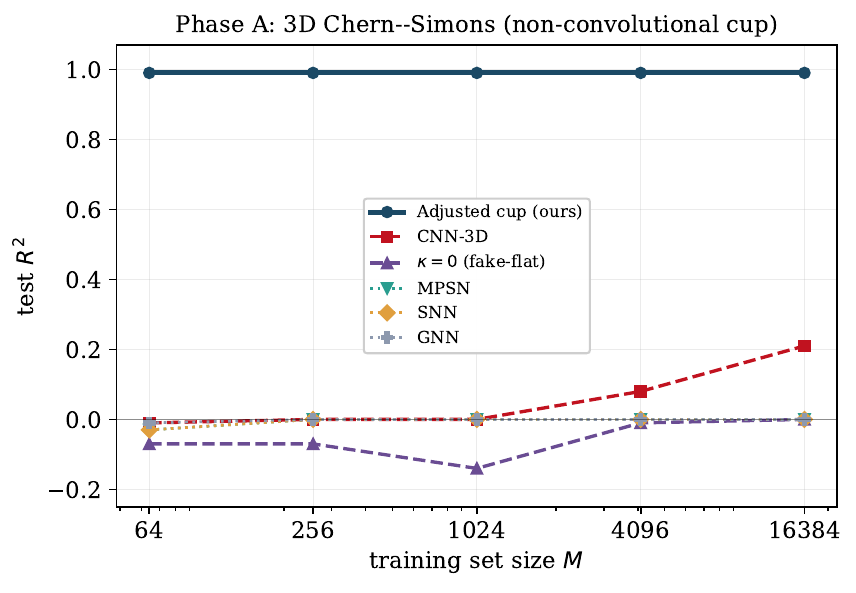}
\caption{Test $R^2$ vs.\ training-set size on 3D Chern--Simons. The adjusted layer generalizes from
$M{=}64$; competent baselines memorize the training set (train $R^2\!\to\!1$, not shown) yet stay at
or below zero on test. The $\kappa{=}0$ ablation confirms the adjustment is the source of signal.}
\label{fig:phaseA}
\end{figure}

\paragraph{4D topological charge (non-convolutional cup).}
Table~\ref{tab:bc} (right): the adjusted layer reaches test $R^2\approx 0.91$; the genuine $4$D CNN
memorizes perfectly (train $1.00$) and generalizes \emph{nothing} (test $0.00$ at every $M$); the
$\kappa{=}0$ ablation is at chance.

\paragraph{Kinetic helicity (convolution-expressible cup).}
Table~\ref{tab:bc} (left) is the honest boundary in action. Because helicity is a convolution
(Proposition~\ref{prop:hel}), the fair CNN \emph{catches up}: test $R^2$ climbs $0.92\to0.99$ by
$M{=}1024$. The adjusted layer is there at $M{=}256$ ($0.99$); the advantage is sample efficiency,
not necessity. We report this rather than hide it.

\begin{table}[t]
\centering
\caption{\textbf{Left: kinetic helicity} ($L{=}16$, a convolution-expressible cup): the fair CNN
catches up. \textbf{Right: 4D topological charge} ($L{=}6$, a non-convolutional cup): the CNN
memorizes but generalizes nothing. Best train\,/\,test $R^2$.}
\label{tab:bc}
\small
\begin{tabular}{rccc@{\hskip 1.5em}rccc}
\toprule
\multicolumn{4}{c}{\textbf{Helicity} (Phase B)} & \multicolumn{4}{c}{\textbf{4D charge} (Phase C)}\\
\cmidrule(r){1-4}\cmidrule(l){5-8}
$M$ & ADJUSTED & FAKE-FLAT & CNN-3D & $M$ & ADJUSTED & FAKE-FLAT & CNN-4D \\
\midrule
256   & +1.00\,/\,+0.99 & +0.06\,/\,$-$0.01 & +1.00\,/\,+0.92 & 256  & +0.88\,/\,\textbf{+0.91} & +0.05\,/\,+0.01 & +1.00\,/\,+0.00 \\
1024  & +1.00\,/\,+1.00 & +0.00\,/\,$-$0.00 & +0.98\,/\,+0.97 & 1024 & +0.91\,/\,\textbf{+0.92} & +0.00\,/\,$-$0.00 & +1.00\,/\,+0.00 \\
4096  & +1.00\,/\,+1.00 & +0.00\,/\,+0.00 & +0.99\,/\,+0.99 & 4096 & +0.90\,/\,\textbf{+0.91} & +0.00\,/\,+0.00 & +0.96\,/\,$-$0.00 \\
16384 & +0.99\,/\,+0.99 & +0.00\,/\,+0.00 & +0.99\,/\,+0.99 &      &                 &                 &                 \\
\bottomrule
\end{tabular}
\end{table}

\subsection{An external benchmark: 2D Chern numbers (Haldane bundles)}
We reproduce the published Haldane-bundle benchmark \cite{haldane}, predicting the integer Chern
number of complex line bundles on the $2$-torus. The adjusted layer (here the Fukui--Hatsugai--Suzuki \cite{fhs}
plaquette cup) is a $0$-parameter structural operator whose accuracy is a \emph{discretization
ceiling} that rises with grid resolution: $83\%\to93\%\to96\%$ at $L=64/96/128$ (Table~\ref{tab:hald},
Figure~\ref{fig:hald}). The $\kappa{=}0$ ablation sits at the majority rate ($16\%$).

We deliberately do \emph{not} claim to beat the benchmark's reported $31\%$ off-the-shelf number: our
own \emph{fair} CNN reaches $66\%$ at $L{=}128$. Instead, the decisive evidence is a
\textbf{gauge-scramble} test (Figure~\ref{fig:hald}, right): applying a random per-site $U(1)$ gauge
transformation to the inputs---which leaves the Chern number unchanged---collapses the CNN to the
majority class ($66\%\!\to\!15\%$) while leaving the gauge-invariant cup essentially unchanged
($93\%\!\to\!95\%$). This isolates gauge-invariance as the operative property in a single controlled
comparison.

\begin{table}[t]
\centering
\caption{\textbf{2D Chern number} (Haldane bundles \cite{haldane}). Test accuracy (exact integer).
ADJ-CUP is a $0$-parameter structural operator; its accuracy is a discretization ceiling rising with
$L$. The bottom row applies a random gauge scramble to the inputs.}
\label{tab:hald}
\small
\begin{tabular}{lccccc}
\toprule
setting & ADJ-CUP & FAKE-FLAT & fair CNN (train/test) & majority \\
\midrule
$L=64$  & 83.0\% & 16.1\% & 83.3\%\,/\,45.6\% & 15.2\% \\
$L=96$  & 92.8\% & 16.1\% & 75.6\%\,/\,55.9\% & 15.2\% \\
$L=128$ & \textbf{96.2\%} & 16.1\% & 89.6\%\,/\,66.2\% & 15.2\% \\
\midrule
$L=96$, \emph{gauge-scrambled} & \textbf{94.8\%} & 15.0\% & 71.5\%\,/\,\textbf{15.0\%} & 14.1\% \\
\bottomrule
\end{tabular}
\end{table}

\begin{figure}[t]
\centering
\includegraphics[width=0.92\textwidth]{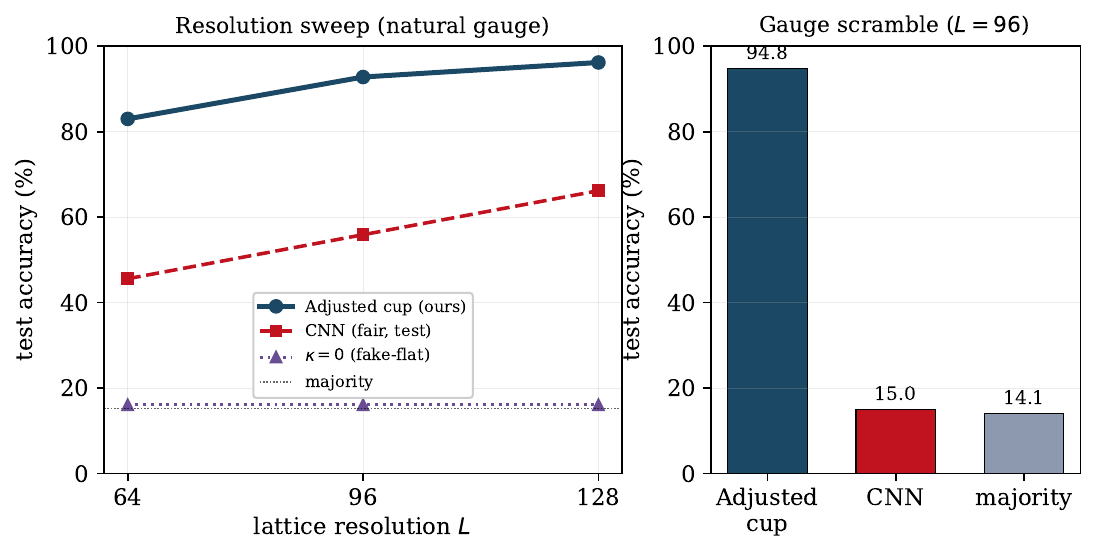}
\caption{\textbf{Left:} 2D Chern accuracy vs.\ grid resolution; the cup rides a discretization ceiling
while the fair CNN trails. \textbf{Right:} under a random local gauge scramble the CNN collapses to
chance while the gauge-invariant cup is unaffected---the cleanest isolation of the mechanism.}
\label{fig:hald}
\end{figure}

\subsection{Robustness to noise and distribution shift}
A recurring objection to neural estimators of topological observables is that closed-form estimators
already exist. We therefore ask where a \emph{learned} component earns its place. We study kinetic
helicity $H=\langle u,\nabla\times u\rangle$, for which spectral helicity is the \emph{exact} operator,
so the only thing separating ground truth from an estimator is corruption of the input. We compare three
estimators, all calibrated or trained at a single corruption level (Gaussian sensor noise $\sigma{=}0.3$
with $50\%$ of sensors observed at random) and then tested under heavier corruption: (a) the
\emph{analytical estimator}---the exact operator applied to the corrupted field, scale-calibrated on
train, no learning; (b) the \emph{cup-net}---a small learned denoiser $\phi_\theta$ (a few-channel
periodic 3D convolution, \emph{zero-initialized so it starts exactly at the analytical estimator})
feeding the \emph{fixed} helicity readout; and (c) a generic \emph{CNN} with no fixed readout. The
spectral operator is verified against an analytic Beltrami flow ($\nabla\times u=u$ to $5\times10^{-7}$;
$H=\lVert u\rVert^2$).

\begin{table}[t]
\centering
\caption{\textbf{Helicity under corruption} (5 seeds, mean$\pm$std). Test $R^2$; methods are tuned at
$\sigma{=}0.3$, $50\%$ sensors, correlation length $\ell{=}1.0$, then tested out-of-condition. Best in
each row in bold.}
\label{tab:robust}
\small
\begin{tabular}{lccc}
\toprule
test condition & Analytical estimator & Cup-net (ours) & CNN \\
\midrule
in-distribution (train cond.) & $0.98{\pm}0.00$ & $\mathbf{0.99{\pm}0.00}$ & $0.99{\pm}0.00$ \\
noise $\sigma{=}0.8$ & $0.93{\pm}0.00$ & $\mathbf{0.99{\pm}0.00}$ & $0.99{\pm}0.00$ \\
\textbf{noise $\sigma{=}1.5$} & $0.61{\pm}0.02$ & $\mathbf{0.94{\pm}0.01}$ & $0.94{\pm}0.01$ \\
\textbf{OOD physics ($50\%$ sensors)} & $0.96{\pm}0.01$ & $\mathbf{0.98{\pm}0.00}$ & $0.91{\pm}0.02$ \\
OOD physics ($25\%$ sensors) & $0.42{\pm}0.01$ & $0.54{\pm}0.03$ & $\mathbf{0.70{\pm}0.02}$ \\
sparse $25\%$ sensors & $0.43{\pm}0.00$ & $0.52{\pm}0.03$ & $\mathbf{0.65{\pm}0.01}$ \\
sparse $10\%$ sensors & $0.07{\pm}0.00$ & $0.11{\pm}0.01$ & $\mathbf{0.18{\pm}0.01}$ \\
\bottomrule
\end{tabular}
\end{table}

Table~\ref{tab:robust} and Figure~\ref{fig:robust} show three regimes. \textbf{(1) Noise.} As noise
grows past the trained level the analytical estimator degrades sharply ($0.98\!\to\!0.61$ at
$\sigma{=}1.5$), because the curl amplifies high-frequency noise; \emph{a learned front-end fixes this}.
The cup-net holds at $0.94$---but so does the CNN ($0.94$): both learned models are far more
noise-robust than the raw estimator, and they tie. The gain here comes from \emph{learning to denoise},
not from the cup specifically. \textbf{(2) Out-of-distribution physics.} On a turbulence spectrum unseen
in training the cup-net stays at $0.98$, matching the analytical estimator's physical generalization
($0.96$) because its readout is the fixed cup, not a learned map; the CNN trails at $0.91$. The cup-net
edge over the CNN is real but modest, not a collapse. \textbf{(3) Extreme sparsity (an honest
limitation).} When sensors are so sparse the operator can barely be evaluated, the
physically-constrained cup-net and analytical estimator degrade together and the \emph{unconstrained}
CNN wins ($0.65$ vs.\ $0.52$ at $25\%$): here the inductive bias is a liability. Overall, the cup-net is
the only method simultaneously noise-robust (unlike the raw estimator) \emph{and} OOD-robust (matching
the estimator, slightly ahead of the CNN), but the margin over a well-tuned CNN is small and reverses
under data sparsity. Helicity is a convolution-expressible cup (Definition~\ref{def:conv}), so this
experiment probes the value of the learned \emph{front-end}; whether the same profile holds for
non-convolutional cups is left to future work.

\begin{figure}[t]
\centering
\includegraphics[width=0.86\textwidth]{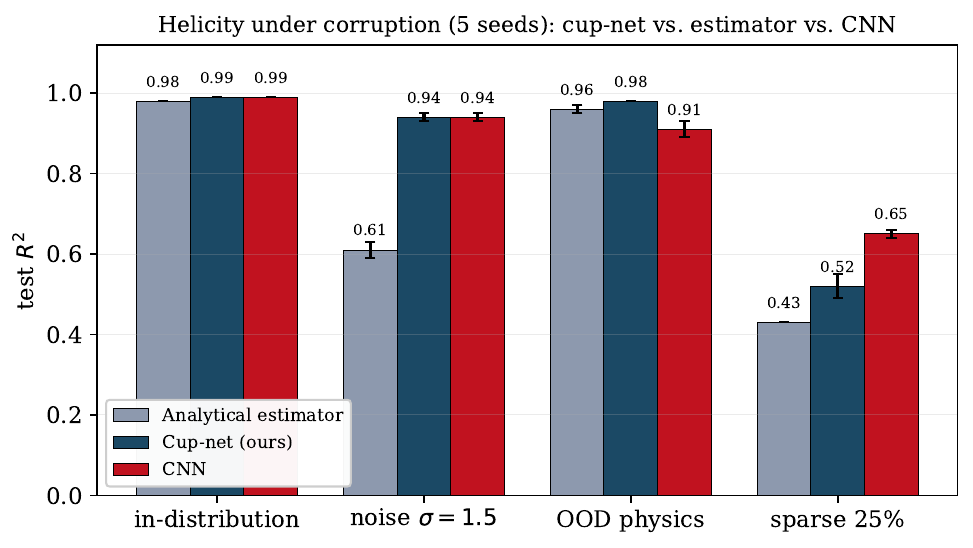}
\caption{\textbf{Helicity under corruption} (5 seeds, error bars $=$ std). A learned front-end
(cup-net or CNN) is far more noise-robust than the raw analytical estimator. The cup-net additionally
retains the estimator's out-of-distribution-physics generalization, slightly ahead of the CNN; under
extreme sensor sparsity the unconstrained CNN wins. The cup-net combines noise- and OOD-robustness, but
the margin over a tuned CNN is modest.}
\label{fig:robust}
\end{figure}

\subsection{A different domain: linking numbers of 3D curves}
To test generality beyond physics we predict the integer \emph{linking number} of two closed $3$D
curves (ground truth: the Gauss linking integral). Each curve is rasterized to a divergence-free
current field; the cup is the mutual-helicity pairing $\int A_1\cdot B_2$ whose non-local inverse-curl
$k/|k|^2$ is the analogue of the adjustment. The structural cup attains accuracy $89\%/94\%$ ($R^2\;
0.75/0.77$) at $L=32/48$; removing the inverse-curl (the ``fake-flat'' analogue) drops it to the
majority rate, and a $3$D CNN on the same fields does not exceed chance (Table~\ref{tab:d4}).

\begin{table}[t]
\centering
\caption{\textbf{3D linking number} vs.\ the Gauss integral. Accuracy / $R^2$. The non-local
inverse-curl (the adjustment analogue) is necessary; a CNN on the same current fields stays at chance.}
\label{tab:d4}
\small
\begin{tabular}{lcccc}
\toprule
$L$ & ADJ-CUP (acc / $R^2$) & FAKE-FLAT & CNN & majority \\
\midrule
32 & \textbf{89.0\%} / +0.75 & 46.4\% / +0.15 & 42.4\% / $-$0.07 & 40.6\% \\
48 & \textbf{93.6\%} / +0.77 & 45.9\% / +0.06 & 42.4\% / $-$0.02 & 40.6\% \\
\bottomrule
\end{tabular}
\end{table}

\subsection{Non-abelian extension and relation to GEBLNet}
Multi-band Chern numbers use the \emph{non-abelian} Wilson-loop holonomy $W\in U(N)$; the
Fukui--Hatsugai--Suzuki invariant $\tfrac{1}{2\pi}\sum\mathrm{Im\,Tr\,Log}\,W$ is the non-abelian
plaquette-holonomy cup. The hard-wired cup (\emph{CUP-NA}) is a $0$-parameter structural \emph{ceiling}
that recovers the integer Chern number exactly ($100\%$ for $N=2,4,7$ bands), \emph{by construction}.
We stress that this is not a learning result and not a claim to outperform the learned
gauge-equivariant network GEBLNet \cite{geblnet}: GEBLNet learns a class function of \emph{local}
Wilson loops, whereas the cup encodes the \emph{non-local} pairing $dA\cupp A$; the two are
complementary. The informative comparison is between \emph{learned} models: a gauge-equivariant
network succeeds where a non-equivariant MLP fails for $N\ge 4$ (Table~\ref{tab:mb}), corroborating
that gauge structure is the operative bias.

\begin{table}[t]
\centering
\caption{\textbf{Multi-band Chern numbers} (U($N$) Wilson loops, $5{\times}5$). CUP-NA is the
$0$-parameter structural \emph{ceiling} (the label function), not a learned competitor; the genuine
comparison is the learned GEBLNet vs.\ a non-equivariant MLP. (Our GEBLNet is a partial reimplementation;
see text.)}
\label{tab:mb}
\small
\begin{tabular}{ccccc}
\toprule
bands $N$ & CUP-NA (ceiling) & GEBLNet (learned) & MLP (non-equiv.) & majority \\
\midrule
2 & 100.0\% & 97.5\% & 71.4\% & 23.4\% \\
4 & 100.0\% & 81.4\% & 28.2\% & 20.9\% \\
7 & 100.0\% & 40.1\% & 21.8\% & 19.4\% \\
\bottomrule
\end{tabular}
\end{table}

\section{Related Work}
\paragraph{Topological and simplicial networks.} Simplicial and topological neural networks process
signals on complexes via the Hodge Laplacian (spectral SNN \cite{ebli}) or via cross-dimensional
message passing (MPSN/SIN \cite{bodnar}), with extensions to trajectory prediction \cite{roddenberry},
$E(n)$-equivariance \cite{battiloro}, and copresheaf structure \cite{copresheaf}. These architectures
are \emph{linear/diffusive} within and across cochain degrees; they do not contain a cup product, and
in our experiments they memorize but do not generalize the oriented cup. Cup products themselves have
been studied in topological data analysis as \emph{persistent} invariants---persistent cup-length and
cup-product structures and their efficient computation \cite{conti_cuplength,memoli_cup,dey_cup}---but
as descriptors extracted from data, not as a differentiable layer inside a network.

\paragraph{Gauge-equivariant learning.} Gauge- and group-equivariant CNNs encode the relevant
symmetry and \emph{learn} the invariant: the icosahedral and homogeneous-space theories
\cite{cohen_ico,cohen_general}, coordinate-independent and higher-order constructions on manifolds
\cite{weiler_ci,he_higher}, and lattice gauge-equivariant CNNs and GNNs for field theory
\cite{favoni,gegnn}. Closest to our targets, GEBLNet \cite{geblnet} learns class functions of
\emph{local} Wilson-loop holonomies for Chern numbers, and earlier work trains networks to predict
topological invariants of band insulators directly \cite{zhang,ohtsuki,haldane}. Our contribution is
orthogonal and complementary: we \emph{encode} the non-local cup pairing $dA\cupp A$ and identify the
adjustment as the term that makes it gauge-invariant.

\paragraph{Higher gauge theory.} The mathematical adjustment originates in higher gauge theory
($2$-connections on $2$-bundles \cite{baezschreiber}) and its adjusted parallel transport and adjusted
higher Chern--Simons formulations \cite{adjtransport,adjcs}; to our knowledge it has not previously
appeared in machine learning.

\section{Conclusion and Limitations}
We introduced a neural layer that hard-wires the adjusted cup product, proved that the adjustment is
the sole source of gauge-invariant signal, and showed across six settings that this bias is
\emph{necessary} for non-convolutional cups and merely \emph{efficient} for convolution-expressible
ones. \textbf{Limitations.} (i) The construction is abelian; the non-abelian extension is realized
only through the structural Wilson-loop cup, not a learned non-abelian layer. (ii) Our targets are
synthetic or external benchmarks with known estimators; the robustness experiment (Table~\ref{tab:robust}) gives a first
controlled demonstration that the learned-frontend cup-net beats the analytical estimator under noise
and distribution shift, but a setting with genuinely measured data where existing estimators fail
(e.g.\ noisy lattice-QCD topological charge or sparse-sensor helicity from a turbulence database)
remains the most valuable next step. (iii) Our GEBLNet is a partial reimplementation and should not be
read as a quantitative comparison. (iv) Several results are $0$-parameter structural operators (Chern,
linking, multi-band); the genuine \emph{learning} claims rest on the trained Chern--Simons, helicity,
and charge experiments and on the gauge-scramble test. (v) The robustness benefit
(Table~\ref{tab:robust}) is shown on a convolution-expressible cup (helicity), so it reflects the value
of the learned \emph{front-end} rather than the cup specifically: a learned denoiser beats the raw
estimator under noise, but the cup-net's margin over a well-tuned CNN is modest (it ties on noise and is
slightly ahead on out-of-distribution physics) and reverses under extreme sensor sparsity, where the
fixed-readout constraint is a liability. The robustness profile for non-convolutional cups is untested.

\end{document}